\documentclass[lettersize,journal]{IEEEtran}
\usepackage{amsmath,amssymb,amsfonts}
\usepackage{array}
\usepackage[caption=false,font=normalsize,labelfont=sf,textfont=sf]{subfig}
\usepackage{textcomp}
\usepackage{stfloats}
\usepackage{url}
\usepackage{verbatim}
\usepackage{graphicx}
\usepackage{cite}

\usepackage{booktabs}
\usepackage{tabularray}
\usepackage{algpseudocode}
\usepackage{multirow}
\usepackage{textcomp}
\usepackage{comment}
\usepackage[longend,ruled,vlined]{algorithm2e}
\usepackage{xcolor, changepage}
\usepackage{pifont}
\usepackage{microtype}
\usepackage{orcidlink}

\newtheorem{theorem}{Theorem}
\newtheorem{lemma}{Lemma}
\newtheorem{proof}{Proof}

\begin{document}

\title{Similarity-based Outlier Detection for Noisy Object Re-Identification\\Using Beta Mixtures}

\author{Waqar Ahmad\orcidlink{0009-0002-4985-6582}, 
    Evan Murphy\orcidlink{0009-0007-7121-2982}, 
    and Vladimir A. Krylov\orcidlink{0000-0002-9734-5974}, 
    {\it Senior Member, IEEE}
\thanks{
 The authors are with the
 School of Mathematical Sciences,
 Dublin City University,  
 Dublin, Ireland.}}

\markboth{}%
{...}

\maketitle

\begin{abstract}
Object re-identification (Re-ID) methods are highly sensitive to label noise, which typically leads to significant performance degradation. 
We address this challenge by reframing Re-ID as a supervised image similarity task and adopting a Siamese network architecture trained to capture discriminative pairwise relationships. Central to our approach is a novel statistical outlier detection (OD) framework, termed Beta-SOD (Beta mixture Similarity-based Outlier Detection), which models the distribution of cosine similarities between embedding pairs using a two-component Beta distribution mixture model. 
We establish a novel identifiability result for mixtures of two Beta distributions, ensuring that our learning task is well-posed.
The proposed OD step complements the Re-ID architecture combining binary cross-entropy, contrastive, and cosine embedding losses that jointly optimize feature-level similarity learning.
We demonstrate the effectiveness of Beta-SOD in de-noising and Re-ID tasks for person Re-ID, on CUHK03 and Market-1501 datasets, and vehicle Re-ID, on VeRi-776 dataset. Our method shows superior performance compared to the state-of-the-art methods across various noise levels (10-30\%), demonstrating both robustness and broad applicability in noisy Re-ID scenarios. The implementation of Beta-SOD is available at: \url{https://github.com/waqar3411/Beta-SOD}.
\end{abstract}


\section{Introduction}

Object re-identification (Re-ID) is a fundamental task in computer vision that seeks to match instances of the same object across different cameras or viewpoints. Person Re-ID has been extensively studied, tackling pose variation, occlusion, illumination, and background clutter~\cite{36}. Equally,  vehicle Re-ID has gained prominence in intelligent transportation systems by confronting inter-class similarity, intra-class variability, and occlusion challenges~\cite{37}. Both subfields fundamentally aim to learn discriminative features that generalize across varied conditions. A further complication arises in practical systems from noisy labels: in object Re-ID, even small annotation errors due to mislabelled identities or detector misalignment can significantly distort learned representation boundaries. Indeed, recent work demonstrates that learning robust person Re-ID models under label noise poses a significant challenge, especially when each identity has limited samples, and naïve training can result in large degradation in accuracy and generalisability~\cite{22,23}.

To address the negative impact of label noise on deep Re-ID models, a variety of strategies have been developed in recent years. One line of work focuses on robust loss functions, since conventional objectives such as cross-entropy (CE) or mean squared error (MSE) are highly sensitive to mislabeled samples~\cite{1,2,3}. A second line of research assigns adaptive importance weights to training examples, down-weighting the contribution of noisy labels while emphasizing reliable ones~\cite{4,5,6}. Another family of approaches explicitly identifies mislabeled samples and either corrects them or removes them from training, thereby purifying the dataset~\cite{7,8,9}. Collectively, these techniques aim to mitigate the effects of corrupted annotations, and improve the stability and generalization of Re-ID systems under realistic conditions, with loss-function design playing a particularly central role.

Among these strategies, the choice of loss function is especially critical, as it directly determines how the model learns from both clean and noisy samples. Standard CE loss, while common, is highly sensitive to label noise because of its strong penalties on misclassifications, often leading Convolutional Neural Networks (CNNs) to overfit corrupted labels~\cite{10,11}. Similarly, contrastive loss (CL), commonly applied in Re-ID tasks to enforce distance-based similarity constraints, is vulnerable to noisy annotations since mislabeled pairs are penalized even when the underlying match is correct~\cite{12,13}. Alternatives such as cosine-embedding loss replace Euclidean distance with angular similarity, offering improved robustness in high-dimensional spaces, yet they too remain affected by label noise due to their dependence on angular separation~\cite{14,15}. To overcome these limitations, several robust loss formulations have been proposed, including generalized cross-entropy, bootstrapped loss, and noise-corrected contrastive objectives, which aim to reduce sensitivity to corrupted labels while preserving discriminative feature learning~\cite{16,17,18}. Nonetheless, loss-based remedies address the complexities of noisy Re-ID data only partially, motivating complementary approaches that filter or adjust corrupted samples through statistical modeling.

In our work we reformulate the Re-ID task as a binary image similarity problem by pairing images into similar and dissimilar pairs, and processing them through a Siamese network~\cite{melekhov2016siamese}. Cosine similarity scores generated by the network are modeled using an explicit statistical model defined as a two-component Beta distribution mixture, where one component represents clean image pairs similarity scores, and the other captures noisy or mislabeled ones. An Expectation-Maximization (EM)-based outlier detection (OD) algorithm identifies and filters noisy pairs. This OD step is alternated with network training every few epochs, allowing the model to iteratively refine the training data and improve task-specific representation learning. Multi-view pairing ensures that enough data is retained even after filtering, which is particularly useful in limited-data scenarios. To enhance learning, we employ a multi-loss strategy that enables the network to capture features beneficial for both similarity classification and Re-ID. The proposed architecture, termed Beta-SOD (Beta mixture Similarity-based Outlier Detection), is outlined in Fig~\ref{fig:proposed_method}.

At the heart of the proposed methodology lies statistical modeling for outlier detection. To express both the signal and noise, or {\it contamination}, components underlying cosine similarity scores, we introduce the use of the Beta distribution~\cite{gupta2004handbook}. This flexible two-parametric family is defined to model quantities restricted to the unit interval $[0,1]$. To establish the theoretical soundness of this formulation and justify the use of EM, we provide, to the best of our knowledge, the first proof of {\it identifiability} for the two-component Beta mixtures. This property means that the individual components can be uniquely recovered from the observed mixture. This result is novel and departs from established identifiability proofs for more classical Gaussian, Gamma~\cite{atienza2006new}, and Generalized Gamma mixtures~\cite{li2015unsupervised}, because in the Beta distribution case the property is violated on subsets of Lebesgue measure zero, which necessitates a distinct line of argument introduced here. Importantly, this violation poses no concern in practice, as the probability of encountering such pathological cases is zero.

The main contributions of this work are as follows:
\begin{itemize}
\item We reformulate the object Re-ID task as pairwise image similarity classification
and propose an explicit statistical OD technique, based
on modeling the cosine similarities as two-component Beta distribution mixtures, which is particularly relevant in the presence of noisy labels.

\item We establish a novel identifiability result for mixtures of two Beta distributions, providing the first theoretical guarantee that such models can be uniquely recovered from data with probability 1. This ensures that our proposed filtering framework is rigorously well-posed.

\item We employ three loss functions in our Siamese-based architecture, including binary cross-entropy, contrastive and cosine embedding losses, enabling the model to learn robust features suitable for both similarity classification and object Re-ID tasks.

\item Extensive experiments on benchmark Re-ID datasets, including Market1501, CUHK03, and VeRi-776, under varying noise levels 0-30\% demonstrate that our method consistently achieves state-of-the-art performance.

\end{itemize}

\begin{figure*}[!h]
    \centering
    \includegraphics[scale=0.34]{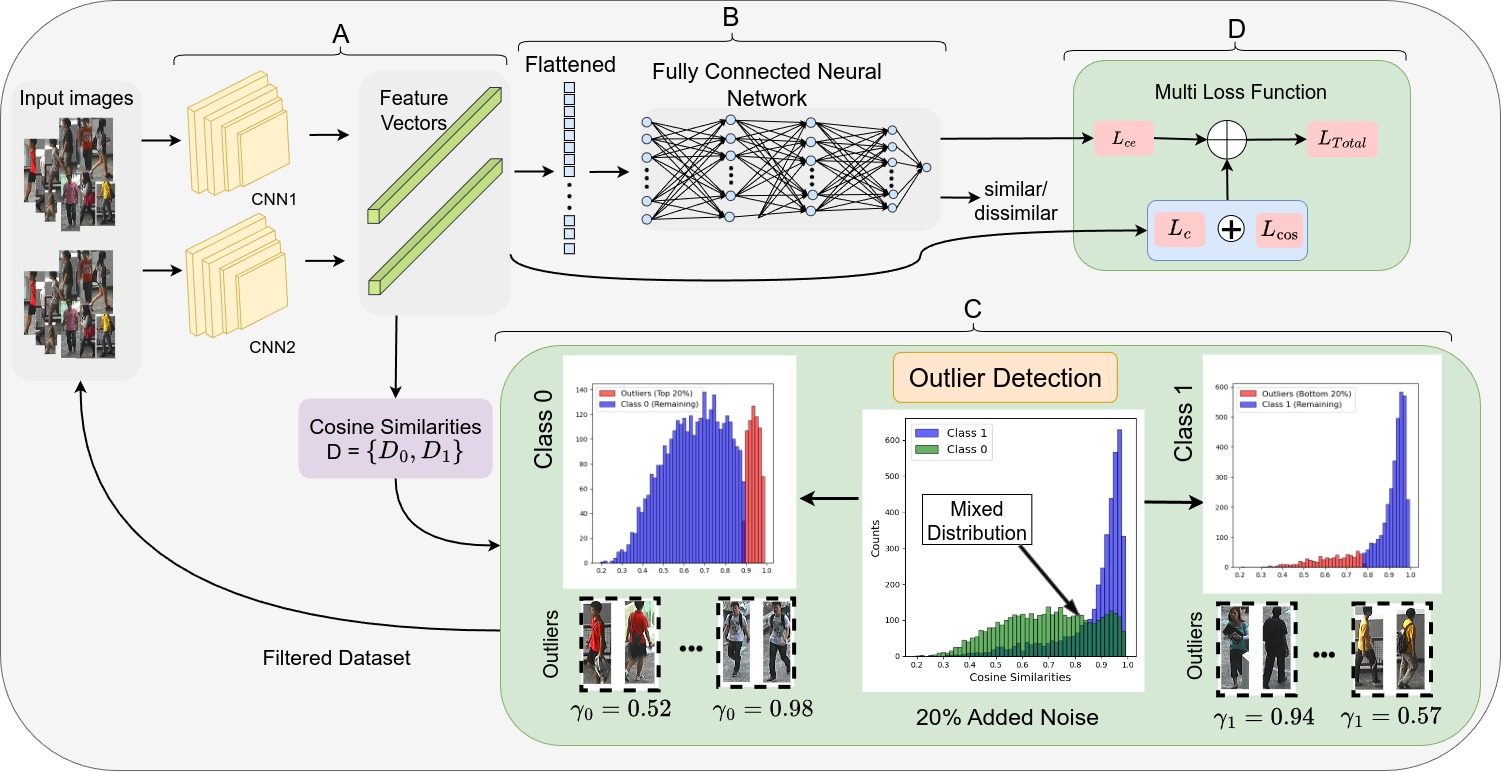}
    \caption{Proposed Beta-SOD Architecture: (A) The Siamese architecture extracts features from paired images. (B) Feature vectors are flattened and passed to a fully connected neural network (FCNN) for pair classification. (C) Cosine similarity is computed from the extracted features and processed by the OD module after every $K$ epochs. The identified outliers are are then filtered from the dataset. (D) The Siamese network and FCNN are trained jointly using the combination of three loss functions.}
    \label{fig:proposed_method}
\end{figure*}


\section{Related work in object Re-ID}
\label{sec:SOA}

Supervised deep learning methods typically rely on large volumes of accurately labeled data, yet in real-world settings collecting such fine-grained person Re-ID annotations is costly and often impractical. Annotations are frequently corrupted by human error, limitations of labeling tools, or dataset re-purposing, where tracking outputs or estimated cross-camera view matches are reused across tasks, introducing identity inconsistencies and mislabeled similarity pairs~\cite{yu2024enhancing,35}. This phenomenon has been empirically documented to degrade performance even under moderate noise levels~\cite{22}. In unsupervised or domain-adaptive Re-ID pipelines, clustering-based pseudo-labels are widespread, but errors in these labels can mislead representation learning unless confidence or sample uncertainty is explicitly modeled~\cite{miao2024confidence}. While CNNs can still learn under moderate noise, even low rates of identity-level corruption severely impair feature discriminability and generalization in Re-ID tasks~\cite{22,23}.

In this work, we focus on the noisy label setting in person re-identification, where annotation errors caused by mislabeling, detector misalignments, or identity ambiguities can severely degrade model performance, particularly in datasets with few samples per identity. To combat this, TSNT~\cite{23} introduces a two-stage noise-tolerant paradigm that first self-refines presumably clean samples and then employs co-training with rectified cross-entropy and robust triplet losses to jointly suppress noisy labels. URNet~\cite{URNet}, though originally proposed for large-scale noisy web image classification, is relevant in this context as it adaptively reweights training instances to mitigate the effects of class imbalance, label noise, and dataset bias. PurifyNet~\cite{20} leverages a hard-aware instance reweighting strategy that gradually purifies noisy annotations while emphasizing challenging but reliable samples. DistributionNet~\cite{19} models feature uncertainty via per-image distributions to mitigate the impact of outliers, though its reliance on richer data per identity may limit generalization in low-sample settings. More recently, CORE~\cite{22} applies collaborative refining via online mutual learning between peer networks with selective consistency to correct labels dynamically during training. Complementing this, Noise-Aware Re-ID~\cite{21} estimates local uncertainty per sample, adaptively down-weighing unreliable data while preserving discriminative features. Additionally, KNN-based filtering techniques~\cite{galanakis2024noise} detect and remove suspicious samples by examining neighborhood consistency in feature space, offering a lightweight noise-robust mechanism.

While these methods target noisy annotations in person Re-ID, similar challenges arise in vehicle Re-ID, where viewpoint changes, domain variations, and annotation inconsistencies require specialized solutions. STE-VReID~\cite{34} proposes deep triplet embedding combined with optimized sampling strategies and achieves competitive performance with minimal complexity. To further strengthen feature learning, GS-TRE~\cite{31} introduces a group-sensitive triplet embedding that explicitly models intra-class variance, improving discrimination under diverse viewpoints. 
Moving beyond supervised paradigms, SGFD~\cite{33} proposes a self-supervised geometric feature discovery framework with interpretable attention, removing reliance on fine-grained labels while still achieving state-of-the-art results. More recently, GiT~\cite{git} integrates graph-based local correction with transformer modules to capture both discriminative details and contextual relations, while GLSD~\cite{GLSD} employs a multi-branch structure with global-local self-distillation to enhance representation quality through adaptive fusion. Extending this line of attention-driven methods, LSKA-ReID~\cite{LSKA-ReID} combines large separable kernel attention with hybrid channel attention to jointly model global structure and fine-grained details, achieving top performance on benchmarks such as VeRi-776~\cite{28}.


\section{Outlier Detection (OD)}
\label{sec:OD}
Noisy labels degrade the learning capability of Re-ID models. From a statistical perspective, image pairs with incorrect or corrupted labels can be regarded as outliers within the dataset. In a typical Re-ID setup, image pairs are constructed and categorized as similar (class 1) if they belong to the same identity, or dissimilar (class 0) otherwise. Given a suitably trained Siamese architecture, similarity is quantified by computing the cosine similarity between the feature vectors extracted from the backbone architecture: 
$\text{sim}(x, y) ={x \cdot y}/(\|x\| \, \|y\|)$. These similarity scores naturally fall within the $[0,1]$ range and can therefore be effectively modeled using a two-parameter Beta distribution, defined by its density function:
\begin{equation}
    f(x; \alpha, \beta) = \frac{\Gamma(\alpha + \beta)}{\Gamma(\alpha)\Gamma(\beta)} x^{\alpha - 1} (1 - x)^{\beta - 1}, \; x\in[0,1],
    \label{eq:betapdf}
\end{equation}
where $\Gamma(\cdot)$ is the Gamma function, and $\alpha, \beta>0$ are model parameters.
The Beta distribution~\cite{gupta2004handbook} is widely used for modeling unknown quantities in the unit interval due to its flexibility in capturing diverse shapes, including uniform, skewed, U-shaped, or unimodal, thus making it well-suited for modeling the similarities. Importantly, this model allows us to avoid imposing unrealistic distribution shape assumptions, like symmetry (Gaussian), or exponential tail decay (Gamma).

In a clean consistent dataset, the histograms of the two classes are well-separated, with minimal overlap and peaks at the opposite ends of the $[0,1]$ interval. However, as can be observed in Fig.~\ref{fig:proposed_method}, the noisy label case presents a mixed pattern. The outliers skew the underlying distribution, causing its mean to shift towards the center of the unit interval. 

We propose to model a noisy distribution as a finite mixture of two Beta distribution components. Note that due to the similarity-based reframing of the Re-ID problem any image pair is coming from either similar or dissimilar class and hence any class histogram is a mix of the pure component plus a contamination component. We model the noisy data as
\begin{equation}
    F_{\text{observed}}(x) = \omega_0 f(x; \alpha_0,\beta_0) + \omega_1f(x; \alpha_1,\beta_1)
    \label{eq:betamix}
\end{equation}
where $(\alpha_0, \beta_0)$ and $(\alpha_1, \beta_1)$ are the parameters of the Beta distribution for class 0 (dissimilar) and class 1 (similar), respectively, and the mixing proportions $\omega_1,\omega_2>0$: $\omega_1+\omega_2=1$. Note that if Eq.~\ref{eq:betamix} models class 0 [class 1] then $\omega_1$ [$\omega_0$] represents the fraction of outliers.

To estimate such mixtures, we adopt the EM algorithm~\cite{dempster1977maximum}. The EM algorithm is particularly well-suited for estimating the parameters of  mixture distributions, where the component memberships of observations are unknown (latent). Direct maximum likelihood estimation is intractable due to the coupling of mixture weights and the Beta distribution parameters, but EM efficiently addresses this by iteratively estimating the posterior probabilities of component assignments (E-step) and updating the mixture weights and Beta parameters (M-step). This allows for stable convergence even in the presence of noisy or overlapping components, making EM a standard and principled approach for mixture modeling in such cases. Alternative approaches include simulation-based Bayesian approaches~\cite{diebolt1994estimation} and Stochastic EM, which are more costly computationally and unnecessary in the considered case, due to the observed convergence of the classical EM.

EM consists of iterating the following two steps applied to a Beta mixture model estimated on dataset $D$:

\noindent {\bf E-Step.}
The posterior probabilities for each sample $x_i$ are computed:
\begin{align}
\gamma_0(x_i) &= P(x_i \in \text{class } 0 \mid x_i; \omega, \alpha, \beta) \notag \\
 &= \frac{\omega_0 \, f(x_i; \alpha_0, \beta_0)}{\omega_0 \, f(x_i; \alpha_0, \beta_0) + \omega_1 \, f(x_i; \alpha_1, \beta_1)} \label{eq:postProb}\\
\gamma_1(x_i) &= P(x_i \in \text{class } 1 \mid x_i; \omega, \alpha, \beta) = 1 - \gamma_0(x_i) \notag
\end{align}
where $\omega_1 = 1-\omega_0$ and $f(x; \alpha, \beta)$ is the density function defined in Eq.~\ref{eq:betapdf}.

\vskip 0.15cm
\noindent {\bf M-Step.} 
The mixture parameters are estimated from the posterior probabilities $\gamma_k$. To this end, we first assign the mixture component $k$ to each sample $x_i$ such that $\gamma_k (x_i) \geq 0.5$. The mixture model parameters are estimated as follows:
\begin{align}
(\alpha_k, \beta_k) &= \text{MLE}_{\text{Beta}}(\{ x_i \mid \gamma_k(x_i) \geq 0.5 \}) \quad k \in \{0, 1\}
\label{eq:alphak}
\end{align}
where Maximum Likelihood Estimation ($\text{MLE}_{\text{Beta}}$) is used to estimate the parameters 
$\alpha_k$ and $\beta_k$ by maximizing the log-likelihood $l(\alpha_k, \beta_k)$ based on the observed data.
The maximization involves the digamma function, which does not allow for a closed-form solution, so the optimization is done numerically using the Newton-Raphson method. Specifically, starting from the trivial initialization (or that from the method of moments using sample mean and variance) we iteratively update the estimates of $\alpha_k$ and $\beta_k$ using the first and second derivatives of the log-likelihood until convergence is achieved.
The updated class prior $\omega_0$ shown in equation \ref{eq:omega}, reflects the proportion of the dataset that is currently believed to the outlier class, and is computed as:
\begin{equation}
    \omega_k = \frac{|\{ x_i \mid \gamma_k(x_i) \geq 0.5 \}|}{|D|}, \quad k \in \{0, 1\}
    \label{eq:omega}
\end{equation}

The noise is assumed to be uniformly distributed across both classes, with roughly 50\% of the noise affecting class 1 and 50\% affecting class 0. Therefore, the 
first run of EM is indented to identify the two large components in the whole dataset comprising the merged contaminated classes 0 and 1. Note that the noise partition does not need to be equal between the classes, it is necessary only that a representative amount of samples is present in both similar and dissimilar classes. This step allows us to use all the available data to assess the statistical behavior of the components, each presented by a substantial number of samples. 
To initialize, $\omega$ is set to 0.5 and $\alpha_0, \beta_0, \alpha_1,\beta_1$ are given to present a right-skewed (1,5) and a left-skewed (5,1) components, respectively. This setup helps to automatically have component 0 associated with the dissimilar class since it's peak is lower, than that of component 1, describing similar pairs and located closer to similarity 1. Beyond that, EM is not sensitive to initialization in the two-component Beta mixture case.

The EM algorithm is subsequently run separately for each of the classes. The estimated parameters from the first EM run are used as initial values. 
These two runs are performed on data which is dominated by the samples from one component: for instance in case of class 1, the data modeled is mostly similar pairs described by $f(x; \alpha_1,\beta_1)$, and a smaller fraction $\omega_0$ of noisy (dissimilar) pairs. In this case $\omega_0<<\omega_1$ and hence a pure EM would typically result in a collapsed performance to $\omega_0=0$. This happens because the noise observations are few and do not present a consistent statistical pattern. We address this undesirable performance by fixing the parameters of the dominant component and estimating solely $\omega$ and the Beta parameters of the noise component.
After running EM on class 0 and 1 we use the resulting estimates of $\omega_1$ and $\omega_0$ to threshold the tails of the two distributions, and eliminate the corresponding image pairs from the training data. This constitutes our OD strategy.

\begin{algorithm}[htbp]
\DontPrintSemicolon
\caption{Outlier Detection (OD) Algorithm for Noisy Image Pairs Removal}
\label{alg:od_algorithm}
\footnotesize 
\SetKwFunction{FEM}{EM}
\SetKwProg{Fn}{Function}{:}{}
\Fn{\FEM{$D$, $\alpha_k$, $\beta_k$, Comp } }{
    {\bf Initialize: } $\omega = \{0.5,0.5\};\; \Delta \omega = 1$;
    
    \While{$\Delta \omega \geqslant 0$}{
        \textbf{E-Step:} Compute $\gamma_0, \gamma_1$ in $D$ using Eqs.~\ref{eq:betapdf},~\ref{eq:postProb};
        
        \textbf{M-Step:} Compute $\omega_k^{new}$, $\alpha_k^{new}$, $\beta_k^{new}$ in $D$ using Eqs.~\ref{eq:alphak},~\ref{eq:omega};

        \tcp{Update both mixture weights}
        $\Delta \omega = \|\omega_k^{new} - \omega_k\|;\;\;\; \omega_k \gets \omega_k^{new}$;
        
        \If{$0 \in Comp$} {
        \tcp{Update component 0}
        $\alpha_0 \gets \alpha_0^{new};\; \beta_0 \gets \beta_0^{new}$;}
        \If{$1 \in Comp$}
        {
        \tcp{Update component 1}
        $\alpha_1 \gets \alpha_1^{new};\; \beta_1 \gets \beta_1^{new}$;
        }
    }
    \KwRet $\omega_k$, $\alpha_k$, $\beta_k$;
  }
\vskip 0.05cm

\KwIn{Datasets $D_s = \{x_i, x_j, s\}_{i,j}^{N_s}$, with $s=0$ dissimilar, and $s=1$ similar image pairs $x_i, x_j$;}

Import pre-trained model;

\For{cycle $=1$ to $C$}{
    Train Siamese model on $D=\{D_0,D_1\}$ for $K$ epochs; 
    
    Compute cosine similarities from the Siamese model;

    \tcp{Learn the outline of components from all data}
    $\omega_k^{*}$, $\alpha_k$, $\beta_k \gets$ \FEM($\{D_0,D_1\}$, $\alpha=\{1,5\}$, $\beta=\{5,1\}$, $\{0,1\}$);

    \tcp{Filter the outliers separately in $D_0$ and $D_1$}

    \tcp{Learn only component 1 (component 0 frozen)}
    $\omega_k$, $\alpha_k^{*}$, $\beta_k^{*} \gets$ \FEM($\{D_0\}$, $\alpha_k$, $\beta_k$, $\{1\}$);

    Update $D_0$ by filtering out highest $\omega_1|D_0|$ samples (outliers) based on cosine similarities;

    \tcp{Learn only component 0 (component 1 frozen)}
    $\omega_k$, $\alpha_k^{*}$, $\beta_k^{*} \gets$ \FEM($\{D_1\}$, $\alpha_k$, $\beta_k$, $\{0\}$);
    
    Update $D_1$ by filtering out lowest $\omega_0|D_1|$ samples (outliers) based on cosine similarities;
}
\vskip .05cm
{\bf Output 1:} Filtered $D=\{D_0,D_1\}$ cleaned of similarity-based outliers;

{\bf Output 2:} Trained Siamese architecture ($CK$ epochs) for Object Re-ID.
\end{algorithm}

\vskip 0.1cm
\noindent{\bf Integration with CNN.} 
We alternate between Siamese training and the OD step, consisting of three EM runs (see Alg.~\ref{alg:od_algorithm}).
The initial fine-tuning (prior to the first round of OD) and the later fine-tuning epochs help ensure that hard inliers are learned by the model and are not prematurely filtered out along with the noisy samples. The training process is terminated after a fixed number of cycles $C$. The overwhelming amount of filtering occurs after cycle 1, with the filtered fraction dropping by a factor of 2-5 with each additional cycle of training.
Fig.~\ref{fig:proposed_method} illustrates the impact of label noise on the similarity score histograms, along with the corresponding detected outliers. The image pairs shown represent samples identified as outliers by the EM-based filtering method, with their respective probabilities indicating likelihood of belonging to each class. This visualization highlights how noise skews the score distribution and how the proposed method effectively isolates mislabeled pairs.

\vskip 0.15cm
\noindent
\textbf{Distribution choice for Cosine Similarities.} The proposed method is based on the Beta distribution, since it provides the most flexible natural fit to the cosine similarity quantities in range $[0,1]$. Fig.~\ref{fig:histograms} shows the quality of fit of this distribution to both classes similarities on the CUHK03 dataset under 0-20\% levels of random noise. We also overlay the fitted probability density functions for Gaussian, and Gamma distributions to evaluate their suitability for modeling the data. The fitting has been achieved following a similar EM procedure relying on MLE to obtain corresponding distribution parameter estimates in the M-step. These two classical families of two-parametric distributions provide worse quality fit largely due to unbounded support: $x\in\mathbb{R}$ for Gaussian, and $x\in\mathbb{R}_{>0}$ for Gamma. As can be observed in Fig.~\ref{fig:histograms}, the increasing noise levels push the two similarity distributions closer together, with the Beta distribution consistently provides better fit, highlighting its robustness for modeling pairwise similarity distributions in noisy settings. Importantly, Gaussian and Gamma mixtures are known to allow the mixing distributions' parameters to be identified uniquely from the observed mixtures~\cite{atienza2006new}, and in Sec.~\ref{sec:Beta} below we prove for the first time that the same can also be said for the Beta mixtures considered in this paper.

\begin{figure*}[!h]
    \centering
    \includegraphics[scale=0.4]{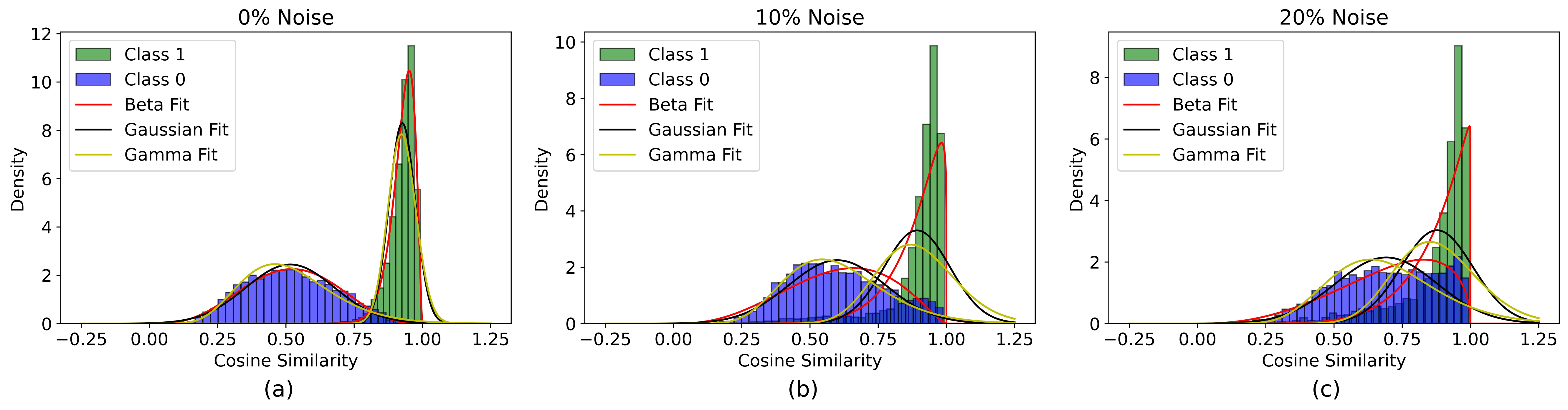}
    \vspace{-.5\baselineskip}
    \caption{Histogram of cosine similarities for both similarity classes on the CUHK03 dataset with overlaid estimated Beta, Gaussian, and Gamma two-component mixtures: with (a) 0\%, (b) 10\% and (c) 20\% random label noise. Unlike Gaussian and Gamma, the Beta distribution models the data in the target range $[0,1]$.}
    \label{fig:histograms}
\end{figure*}


\section{Beta mixtures identifiability}
\label{sec:Beta}

We will demonstrate here that the identification of mixture components designed in Sec.~\ref{sec:OD} is a correctly posed problem in that a unique composition of two Beta distribution components can be recovered. This will be established in Theorem~\ref{th:2} below.

In the following, we refer to Beta distributions $f(x;\alpha,\beta)$
defined in Eq.~\ref{eq:betapdf} for $(\alpha,\beta) \in \mathbb{R}^2_{>0}$ and $x\in[0,1]$.

We first need to establish the behavior of Gamma function around the points $z=0,-1,-2,...$ where $\Gamma(z)$ is infinite~\cite{abramowitz1965}.
\begin{lemma}\label{lem:1}
    For any $n\in\mathbb{Z}_{\geq 0}$, $\lim\limits_{\varepsilon\to 0^+} \varepsilon \Gamma(-n+\varepsilon) = \frac{(-1)^n}{n!}$.
\end{lemma}

\begin{proof}
    Iteratively applying the core fact that $\Gamma(z+1)=z\Gamma(z)$ for any $z \in \mathbb{R}$~\cite{abramowitz1965}, it holds that for any $\varepsilon>0$, and integer $n \geq 0$,
    \begin{align*}
        \Gamma(1+\varepsilon) = \varepsilon(-1+\varepsilon)\cdots(-n+\varepsilon)\Gamma(-n+\varepsilon).
    \end{align*}
    Rearranging, this becomes
    \begin{align*}
        \varepsilon \Gamma(-n+\varepsilon) = \frac{\Gamma(1+\varepsilon)}{(-1+\varepsilon)\cdots (-n+\varepsilon)}.
    \end{align*}
    Then taking limits as $\varepsilon \to 0^+$ yields
    \begin{align*}
        \lim_{\varepsilon\to 0^+} \varepsilon \Gamma(-n+\varepsilon) &= \lim\limits_{\varepsilon \to 0^+} \frac{\Gamma(1+\varepsilon)}{(-1+\varepsilon)\cdots (-n+\varepsilon)}\\
        &= \frac{\Gamma(1)}{(-1)\cdots(-n)} = \frac{(-1)^n}{n!},
    \end{align*}
    as required.
\end{proof}

\begin{theorem} \label{th:1}
    Let $(a_1,b_1),(a_2,b_2) \in \mathbb{R}^2_{>0}$, $\pi \in (0,1)$, and consider the mixture distribution
    \begin{align*}
        F(x) = \pi f(x;a_1,b_1)+(1-\pi)f(x;a_2,b_2).
    \end{align*}
    If $(a_1,b_1)-(a_2,b_2) \notin \mathbb{Z}^2$, or, in other words, either $a_1-a_2$ or $b_1-b_2$ (or both) are not integers, then $F$ has a unique representation as a mixture of two Beta distributions.
\end{theorem}

\begin{proof}
Let $\pi \in (0,1)$, and $(a_1,b_1),(a_2,b_2)\in \mathbb{R}^2_{>0}$ be distinct, such that $(a_1,b_1)-(a_2,b_2) \notin \mathbb{Z}^2$. Without loss of generality, assume that $a_1-a_2 \notin \mathbb{Z}$, and $a_1<a_2$. Indeed, if only $b_1-b_2 \notin \mathbb{Z}$ holds, we can then swap variable to $y:=1-x$ and revert to the $a_1-a_2 \notin \mathbb{Z}$ case due to the symmetry of the Beta distribution.

We will now use the integral Mellin transform defined for any $(\alpha,\beta)\in \mathbb{R}^2_{>0}$ as
\begin{align}
    m_{\alpha,\beta} = \int\limits_0^1 x^{t-1} f(x;\alpha,\beta)\,\text{d}x = \frac{\Gamma(t+\alpha-1)\Gamma(\alpha+\beta)}{\Gamma(t+\alpha+\beta-1)\Gamma(\alpha)}, \label{eq:mellin}
\end{align}
which can also be applied to the mixture $F$, yielding
\begin{align*}
    M_F(t) = \pi m_{a_1,b_1}(t)+(1-\pi)m_{a_2,b_2}.
\end{align*}
Note that the right hand side in Eq.~\ref{eq:mellin} has singularities for $t \in S_{\alpha,\beta}= \{1-\alpha-k:k \in \mathbb{Z}_{\geq 0}\}$ if $\beta \notin \mathbb{Z}_{>0}$, and $t \in S_{\alpha,\beta}= \{1-\alpha-k:k=0,...,\beta-1\}$ if $\beta \in \mathbb{Z}_{>0}$. To see this, notice that when $\beta \notin \mathbb{Z}_{>0}$, $\Gamma(t+\alpha-1)$ is non-finite if and only if $t \in S_{\alpha,\beta}$, and is the only non-finite term appearing in $S_{\alpha,\beta}$ in this case. Since the denominator is never zero due to the properties of $\Gamma$, these are the only values at which the numerator is non-finite, and a singularity exists. When $\beta \in \mathbb{Z}_{>0}$, using the core fact that $\Gamma(z+1)=z\Gamma(z)$ for all $z \in \mathbb{R}$, we can write 
\begin{align*}
\frac{\Gamma(1+\alpha-k)}{\Gamma(1+\alpha+\beta-k)} = \frac{1}{(1+\alpha-k)\cdots(1+\alpha+(\beta-1)-k)}.
\end{align*}
This expression has a singularity precisely when the denominator is zero, which happens if and only if $t \in S_{\alpha,\beta}$.

We now have that $m_{\alpha,\beta}$ converges outside of the set of singularities. Let $S:=S_{a_1,b_1} \cup S_{a_2,b_2}$ be the singularities of $M_F(t)$. Hence
$$
    \max S = \max S_{a_1,b_1} \cup S_{a_2,b_2}
    = \max \{1-a_1,1-a_2\} = 1-a_1.
$$
Thus, parameter $a_1$ can be uniquely recovered from observing the behavior of the mixture $F$ through $M_F(t)$. Moreover, as $S_{a_1,b_1}$ and $S_{a_2,b_2}$ are disjoint by the fact that $a_1-a_2 \notin \mathbb{Z}$, 
\begin{align*}
    \max S\backslash S_{a_1,b_1} = \max S_{a_2,b_2} = 1-a_2.
\end{align*}
Hence the parameter $a_2$ can also be recovered.

We can now address the recovery of $b_1,b_2$. Note that if $b_1 \in\mathbb{Z}_{>0}$ or $b_2 \in \mathbb{Z}_{>0}$, then $|S_{a_1,b_1}| = b_1$ or $|S_{a_2,b_2}| = b_2$, respectively. Also, we may recover $\pi$ by noticing that
\begin{align*}
    \pi = \frac{F(x)-f(x;a_2,b_2)}{f(x;a_1,b_1)-f(x;a_2,b_2)},
\end{align*}
and evaluating the right hand side at any $x \in (0,1)$. Going forward, we assume $(b_1,b_2) \notin \mathbb{Z}^2$. Define the sequence $\{t_n\}_{n\in\mathbb{Z}_{>0}}$, where $t_n = 1-a_1-n$. We now define residuals
$$R_n :=\lim\limits_{t \to t_n^+} (t-t_n)M_F(t).$$
Since $F$ is known (or observed), so is $M_F(t)$ for any $t$, and hence $R_n$ can be computed for each $n$. Moreover, as $m_{a_2,b_2}(t_n)$ is finite for each $n$, we have
\begin{align*}
    R_n &= \lim\limits_{t \to t_n^+} (t-t_n)[\pi m_{a_1,b_1}(t)+(1-\pi)m_{a_2,b_2}]\\
    &= \pi\lim\limits_{t \to t_n^+} (t-t_n)m_{a_1,b_1}(t).
\end{align*}
Making the substitution $\varepsilon = t_n-t$ and using Eq.~\ref{eq:mellin} we get
\begin{align*}
    &R_n = \pi \lim\limits_{\varepsilon \to 0^+} \varepsilon\frac{\Gamma(-n+\varepsilon)\Gamma(a_1+b_1)}{\Gamma(-n+b+\varepsilon)\Gamma(a_1)}\\
    &= \frac{\pi \Gamma(a_1+b_1)}{\Gamma(a_1)}\lim\limits_{\varepsilon \to 0^+} \left[\frac{1}{\Gamma(-n+b_1+\varepsilon)}\right] \lim\limits_{\varepsilon\to 0^+} \left[\varepsilon \Gamma(-n+\varepsilon)\right]  \\
    &=\frac{\pi \Gamma(a_1+b_1)}{\Gamma(a_1)} \frac{1}{\Gamma(-n+b_1)} \frac{(-1)^n}{n!},
\end{align*}
where the last equality uses Lemma \ref{lem:1}. For $n=0,1$, we have
\begin{align*}
    &R_0 = \frac{\pi\Gamma(a_1+b_1)}{\Gamma(a_1)\Gamma(b_1)},\\
    &R_1 = -\frac{\pi \Gamma(a_1+b_1)}{\Gamma(a_1)\Gamma(b_1-1)} = -\frac{\pi(b_1-1)\Gamma(a_1+b_1)}{\Gamma(a_1)\Gamma(b_1)},
\end{align*}
where the last equality follows from the fact that $\Gamma(b_1)=(b_1-1)\Gamma(b_1-1)$. Then notice that
\begin{align*}
    \frac{R_1}{R_0}  = -(b_1-1).
\end{align*}
Thus, we can recover $b_1$ and $\pi$ as
\begin{align*}
    b_1 = 1-\frac{R_1}{R_0}, \quad \pi = R_0 \frac{\Gamma(a_1)\Gamma(b_1)}{\Gamma(a_1+b_1)}.
\end{align*}
The recovery of $b_2$ follows an identical logic, using the sequence $t_n'=1-a_2-n$. Therefore, we have recovered unique values of $a_1,b_1,a_2,b_2$ and $\pi$. Thus, $F$ must have a unique representation as a mixture of two Beta distributions.
\end{proof}

To complement the statement in Theorem~\ref{th:1}, we can notice that violation of the given restriction can indeed result in non-identifiability. For example, one can easily show that
\begin{align*}
    \frac{1}{2}f(x;2,2) + \frac{1}{2} f(x;2,3) = \frac{1}{4} f(x;3,2) + \frac{3}{4}f(x;2,3),
\end{align*}
which occurs since both $a_1-a_2=2-2=0$ and $b_1-b_2=2-3=-1$ on the left hand side are integers, which results in the existence of another distinct pair of Beta distributions that produce an identical mixture.

The property established in Theorem~\ref{th:1} is called {\it identifiability} of mixtures. This means that having observed a mixture distribution, the original mixed components' parameters as well as the mixing proportion are uniquely recoverable. We can thus reformulate the previous result as follows:

\begin{theorem}\label{th:2}
 Let $(a_1,b_1),(a_2,b_2) \in \mathbb{R}^2_{>0}$, $\pi \in (0,1)$, and consider the mixture of two Beta distributions
    \begin{align*}
        F(x) = \pi f(x;a_1,b_1)+(1-\pi)f(x;a_2,b_2).
    \end{align*}
Then with probability $1$, $F(x)$ is identifiable.
\end{theorem}

\begin{proof}
To demonstrate that this result immediately follows from Theorem~\ref{th:1} it suffices to observe that for any $(a_1,b_1)\in\mathbb{R}^2_{>0}$ the probability of encountering such $(a_2,b_2)$ that both $a_1-a_2\in\mathbb{Z}$ and $b_1-b_2\in\mathbb{Z}$ is zero. This is indeed the case since this set $\{(a_1+m,b_1+n)\}_{m,n\in \mathbb{Z}}$, contains countably many elements by construction, and hence has Lebesgue measure zero.
\end{proof}

Theorem~\ref{th:1} provides a novel and significant result for two-component Beta mixtures with an explicit restriction on non-identifiable cases. From the applications perspective, since this result holds with probability 1, we can expect the identifiability to be the case for any practically observed mixture. This means that parameter estimation solved via EM is posed correctly.


\section{Siamese architecture and Loss Functions}
\label{sec:arch}
When dealing with Siamese architectures, contrastive loss is commonly employed, where pairs yielding lower loss values are treated as similar, while those with higher loss values are treated as dissimilar~\cite{24, 25, 26}. Alternatively, similarity can be assessed using cosine similarity between feature vectors, followed by classification using either a predefined threshold or a learnable linear layer~\cite{27}.

In this work, rather than relying solely on contrastive loss or cosine similarity-based loss, we employ a combination of loss functions to optimize the model parameters more effectively. The combination includes binary cross-entropy loss $ (\mathcal{L}_{CE}) $, cosine-embedding loss $ (\mathcal{L}_{COS}) $ and contrastive loss $ (\mathcal{L}_{C})$.
The cross-entropy loss:
\begin{equation}
    \mathcal{L}_{CE} = -\frac{1}{N} \sum_{i=1}^{N} \left[ y_i \log \hat{y}_i + (1 - y_i) \log (1 - \hat{y}_i) \right],
    \label{eq:bce_loss}
\end{equation}
is obtained from a fully connected neural net (FCNN). The FCNN takes the 1280-dimensional vectors as an input and includes four fully connected layers with a binary output layer.

The contrastive loss:
\begin{equation}
    \mathcal{L}_{C} = \frac{1}{N} \sum_{i=1}^{N} \left[ y_i d_i^2 + (1 - y_i) \max(0, m - d_i)^2 \right],
    \label{eq:contrastive_loss}
\end{equation}
where $ d_i $ is the euclidean distance and $ m $ is the margin parameter,
is effective when class boundaries are well separated. However, in applications such as person Re-ID, visually similar individuals, for instance, two people viewed from behind wearing the same clothing color, may produce low Euclidean distances, leading to potential misidentification.
This arises because contrastive loss alone fails to capture subtle appearance differences that are not well represented by Euclidean distance.

To address this, we combine contrastive loss with cosine embedding loss, which is more effective at distinguishing such instances by focusing on angular similarity rather than just distance.
\begin{equation}
\begin{split}
\mathcal{L}_{\text{COS}} = 
\frac{1}{N} \sum_{i=1}^{N}
\begin{cases}
1 - \cos(\mathbf{x}_i, \mathbf{y}_i), &\text{if similar} \\
\max(0, \cos(\mathbf{x}_i, \mathbf{y}_i) - m), &\text{otherwise}
\end{cases}
\end{split}
\label{eq:cosine_embedding_loss}
\end{equation}
where $ \mathbf{x}_i, \mathbf{y}_i $ represent the feature vectors of a similar/dissimilar image pair. $ \cos(\mathbf{x}_i, \mathbf{y}_i) $ is the cosine similarity between feature vectors and $m$ is the margin value, controlling the penalty for high similarity. $\mathcal{L}_{\text{COS}}$ is beneficial when comparing images that are similar in terms of appearance but not in spatial orientation / pose.

The total loss $ \mathcal{L}_{total} $ is the weighted sum of three loss functions:
\begin{equation}
    \mathcal{L}_{total} = \mathcal{L}_{CE} + \lambda \mathcal{L}_{COS} + (1-\lambda) \mathcal{L}_{C}
    \label{eq:combined_loss}
\end{equation}
where $ \lambda $ is the weighting hyper-parameters that balances the contribution of the loss terms.

The overall architecture of the Siamese network with the proposed losses and the OD module is shown in Fig.~\ref{fig:proposed_method}. The OD module is performing filtering of the two datasets with dissimilar and similar image pairs after every $K$ epochs such that the outliers are being identified and filtered out in a task-specific manner.


\section{Experiments}
\label{sec:exp}
\subsection{Experimental Setup}

\noindent {\bf Evaluation Metrics.}
In our work, we focus on noisy-label object Re-ID but also address noise detection. The proposed architecture consists of two components—MobileNetV2 for feature extraction and a FCNN for classifying image pairs similarity. To evaluate the performance of the model, we employ distinct metrics for each task. For classification, we report standard evaluation metrics including test accuracy, precision, and recall. For object Re-ID, we utilize the rank-k matching accuracy and the mean Average Precision (mAP).

The mAP metric provides a comprehensive evaluation of re-identification performance by averaging the precision values across all queries. The Average Precision (AP) for each query, and hence the overall mAP are obtained as follows:
$$
    AP  = \frac{\sum_{k=1}^{n} P(k) gt(k)}{N_{gt}}, \qquad
    mAP = \frac{\sum_{q=1}^{Q} AP(q)}{Q},
$$
where, $k$ denotes the rank position among the $n$ retrieved objects, $N_{gt}$ represents the overall number of relevant (ground-truth) objects, and Q is the total number of query images. $P(k)$ denotes the precision at rank k in the retrieval list, while $gt(k)$ indicates if the  $k$-th retrieved image is a correct match or not.

Rank-k (Rk) measures the proportion of queries for which the correct match appears within the top-k retrieved results:
$$Rk = \frac{\sum_{q=1}^{Q} gt(q,k)}{Q},$$
where $gt(q,k)$ equals 1 when ground truth of q image appears in the top $k$ retrieved samples.

\vskip 0.15cm
\noindent {\bf Datasets.}
The proposed methodology is evaluated on three multi-view Re-ID datasets: CUHK03~\cite{29} and  Market1501~\cite{30} for person Re-ID, and VeRI-776~\cite{28} for cars Re-ID.

\textit{CUHK03} dataset contains 14,097 images of 1,467 identities captured by 5 disjoint camera views. The images are divided into training set with 7,365 images, gallery set with 5,332 images and query set with 1400 image samples. Two type of annotation is provided i.e., manually labeled and detected bounding boxes. We use the latter to enable comparisons with the benchmark methods.

\textit{Market1501} dataset contains 32,668 images  captured by 6 disjoint camera views. The images set is further divided into training+gallery sets containing 12,936+19,732 images of 751 identities, and the query set containing 3,368 images of another 750 identities.

\textit{VeRi-776} dataset contains 51,035 images of 776 vehicles. Each vehicle is captured by up to 18 disjoint camera views. The dataset is provided with different types of annotations including bounding boxes, vehicle types, colors, brands and cross-camera relations. The dataset is further divided into training set with 37,778 images, gallery set with 11,579 images and query set with 1678 images.

As the proposed method classifies image pairs into similar and dissimilar categories, we begin by generating image pairs across all identities and assigning them to similar class (label 1). An equal number of dissimilar pairs (label 0) are created by randomly selecting images from different identities and pairing them together.

\vskip 0.15cm
\noindent {\bf Label Noise Generation.}
We evaluate our methodology under the following two core noise settings, each applying the equal target amount of corrupt labels to both image pairs classes, similar and dissimilar. 
Firstly, we generate {\it Random Noise} datasets. Similarly to~\cite{19,20,23} we introduce noise by randomly (uniformly) selecting images from the dataset and reassigning their labels to different (random) identities. Specifically, class 0 (dissimilar) label noise is introduced by pairing images from the same identity while label noise in class 1 is generated by pairing images from distinct identities.
Secondly, we consider {\it Pattern Noise}, where the datasets is curated to render denoising more challenging~\cite{21,23}. To do so, we trained a ResNet-50 model on a clean dataset. This network was then used to extract feature representations of all image samples. To generate pattern noise labels, we paired samples from different identities that exhibited highest cosine similarity and labeled them as similar pairs (class 1). Conversely, we paired samples of the same identity with lowest cosine similarity (originally belonging to class 1) and mislabeled these as class 0.

\vskip 0.15cm
\noindent {\bf Implementation Details.}
The proposed methodology is implemented in PyTorch using a pre-trained MobileNetV2~\cite{sandler2018mobilenetv2} as the backbone. The similarity FCNN consists of four fully connected layers (512, 512, 256, and 128 neurons) followed by Batch Normalization and ReLU activation. Dropout layers with a 0.3 rate are applied after the last two layers. Data augmentation includes random cropping and horizontal flipping, with image resolutions of 300x150, 128x64, and 150x150 for the CUHK03, Market1501, and VeRi-776 datasets, respectively. A batch size of 32 is used for all datasets, and the Adam optimizer with an initial learning rate of 0.001 decaying every 7 epochs. Training stops once the model stabilizes. The number of epochs $K$ is determined based on validation accuracy; once the validation accuracy stabilizes, the OD process is initiated. We have observed little sensitivity to this parameter having considered a range of 2-8 epochs. Parameter $C$ depends on $w_k$, and OD is terminated when $w_k$ drops below 0.01\%. In general, the selection of $K$ and $C$ should allow for a similar overall number of training epochs $CK$ as the original architecture suggests. The combined loss is used to update model parameters, with $\lambda=0.45$ tuned via grid search.

From a computational perspective, the cost of running Beta-SOD is essentially equivalent to training the underlying Siamese backbone and subsequent FCNN. The additional overhead of the EM step for noise identification and filtering is minimal, requiring less than 10 milliseconds on a CPU. Since the OD module is relevant only for the training phase, the deployment consists in simply running the Siamese backbone + FCNN.

\subsection{Results \& Comparisons}
In this section we present the outcomes of the noise detection, as well as the person and vehicle Re-ID task. We also report comparisons with alternative statistical models for the OD module and ablation studies.

The technique proposed in this work relies on explicit statistical modeling under the assumption that noise represents a contamination to the clean signal component. For this reason, we focus on Beta-SOD evaluation with noise levels at 0-30\%, where from statistical point of view, the ``true'' distribution is dominant. At a 50\% noise level, the histograms of the two classes are practically identical (in the random noise setting), making it unfeasible for the proposed statistical method to perform effectively without additional supervision. We hence choose not to consider such scenarios of noise at 50\% or higher, which require structural modifications and result in relatively poor performance, e.g., under 20\% mAP on the CUHK03 dataset~\cite{23}. Furthermore, we consider such extreme contamination levels to be uncommon in practice, and thus of less interest.

\vskip 0.15cm
\noindent {\bf Siamese backbone selection.}
In Tab.~\ref{table:backbone_CNNs} we present comparisons between several typical Siamese architecture backbones: ResNet-18, ResNet-50~\cite{he2016deep}, and MobileNetV2~\cite{sandler2018mobilenetv2}. Given its size and very competitive accuracy reported, we employ MobileNetV2 as the primary architecture for our model.

\begin{table*}[t]
\caption{Siamese Backbone Network selection on CUHK03 Dataset with various random noise level (5 runs average)}
\resizebox{\textwidth}{!}{%
\centering
\begin{tabular}{c|c|cccc|cccc|cccc}
\hline
\multirow{2}{*}{\begin{tabular}[c]{@{}c@{}}Backbone\\ Networks\end{tabular}} & \multirow{2}{*}{Parameters} & \multicolumn{4}{c|}{0\%} & \multicolumn{4}{c|}{10\%} & \multicolumn{4}{c}{30\%} \\ \cline{3-14} 
 &  & mAP & R1 & R5 & R10 & mAP & R1 & R5 & R10 & mAP & R1 & R5 & R10 \\ \hline
Resnet18 & $\sim$4M & 50.94 & 52.71 & 65.18 & 75.2 & 38.1 & 43.31 & 56.18 & 70.92 & 31.8 & 35.4 & 54.92 & 68.8 \\
Resnet50 & $\sim$23.9M & 58.9 & \textbf{58.7} & \textbf{74.61} & 79.7 & \textbf{48.91} & 52.1 & 70.9 & 77.7 & 45.51 & 50.14 & 70.21 & \textbf{77.1} \\
MobileNet-V2 & $\sim$3.4M & \textbf{59.2} & 58.56 & 74.58 & \textbf{79.82} & 48.33 & \textbf{52.61} & \textbf{71.82} & \textbf{77.75} & \textbf{46.8} & \textbf{50.2} & \textbf{71.01} & 76.9 \\ \hline
\end{tabular}}
\label{table:backbone_CNNs}
\end{table*}

\begin{table*}[ht]
\centering
\caption{Noise identification performance with Various Random Noise Levels (5 runs average)}
\begin{tabular}{c|cccc|cccc|cccc}
\hline
\multirow{2}{*}{Noise} & \multicolumn{4}{c|}{Market1501} & \multicolumn{4}{c|}{CUHK03} & \multicolumn{4}{c}{VeRi-776} \\ \cline{2-13} 
 & Outliers & Precision & Recall & Accuracy & Outliers & Precision & Recall & Accuracy & Outliers & Precision & Recall & Accuracy \\ \hline
0\% & 0.5 & - & - & 94.8 & 1.62 & - & - & 89.5 & 1.3 & - & - & 93.89 \\
10\% & 10.31 & 85.79 & 76.87 & 92.1 & 11.32 & 82.09 & 73.63 & 87.21 & 11.34 & 87.46 & 84.8 & 92.1 \\
20\% & 21.26 & 82.93 & 80.56 & 91.8 & 21.56 & 80.47 & 76.08 & 85.1 & 22.76 & 85.66 & 88.39 & 90.83 \\
30\% & 32.97 & 75.73 & 81.67 & 90.7 & 32.28 & 75.75 & 81.18 & 84.2 & 33.05 & 84.47 & 89.86 & 89.76 \\ \hline
\end{tabular}
\label{tab:denoising_results}
\end{table*}

\begin{table*}[!t]
\caption{
Noise Identification performance of the OD based on Gaussian, Gamma and Beta mixtures with various random noise levels on CUHK03 Dataset (5 runs Average)}
\resizebox{\textwidth}{!}{%
\begin{tabular}{c|cccc|cccc|cccc}
\hline
\multirow{2}{*}{Noise} & \multicolumn{4}{c|}{Gaussian mixture} & \multicolumn{4}{c|}{Gamma mixture} & \multicolumn{4}{c}{Beta mixture} \\ \cline{2-13}
 & Outliers & Precision & Recall & Accuracy & Outliers & Precision & Recall & Accuracy & Outliers & Precision & Recall & Accuracy \\ \hline
0\% & 2.13 & - & - & 87.71 & 4.51 & - & - & 87.3 & 1.62 & - & - & 89.5 \\
10\% & 13.18 & 76.01 & 71.99 & 85.62 & 13.09 & 72.4 & 70.95 & 83.92 & 11.33 & 82.28 & 73.88 & 87.12 \\
20\% & 24.14 & 72.15 & 75.7 & 81.55 & 25.06 & 64.42 & 75.93 & 77.6 & 21.56 & 80.32 & 76.19 & 85.4 \\ \hline
\end{tabular}
}
\label{tab:BGG_outlier}
\end{table*}

\vskip 0.15cm
\noindent {\bf Noise filtering performance.}
We first assess the OD module's performance 
in identifying the contamination, i.e. incorrectly labeled (noisy) image pairs, in the dataset of cosine similarities. As reported in Tab.~\ref{tab:denoising_results},
the proposed model achieves 94.8\%, 89.5\%, and 93.89\% similarity classification accuracy on clean Market1501, CUHK03, and VeRi-776 datasets. At 30\% noise level, these accuracies decline by 4.1\%, 5.3\%, and 4.13\%. 
Increasing the noise level from 10\% to 30\% results in slightly more detected outliers, indicating that the proposed denoising method effectively filters noisy pairs with a tendency towards overshooting. At higher noise levels (30\%), precision decreases slightly due to more false positives, while recall improves as more true positives are retained and fewer false negatives occur. Nevertheless, precision remains around 76\% for Market1501 and CUHK03, and 84\% for VeRi-776, demonstrating the robustness of the proposed method against label noise. This degradation can be primarily attributed to the removal of challenging positive or negative examples during the denoising process, which somewhat reduces the model’s ability to generalize. A small subset of challenging noisy samples persist after filtering, further contributing to the similarity classification decline. 

\vskip 0.15cm
\noindent {\bf Distribution modeling.}
To further evaluate the suitability of the Beta distribution for modeling cosine similarities, we compare the proposed Beta model with Gamma and Gaussian distributions for the OD on the CUHK03 dataset, which is summarized in Table~\ref{tab:BGG_outlier}.The Beta distribution achieves the highest classification accuracy of 89.5\% on the clean dataset, whereas the relatively lower scores of the Gaussian and Gamma distributions.  Moreover, as the noise level increases to 20\%, the amount of over-filtering reaches 4–5\% for Gamma and Gaussian distributions, whereas the Beta model shows substantially lower overshooting of only 1–2\%, confirming its robustness in noisy conditions. Qualitatively, in Fig.~\ref{fig:histograms} we can observe that both Gaussian and Gamma distributions do not fit the data as effectively as the Beta distribution. This over-filtering of challenging pairs negatively affects dataset composition and, consequently, reduces model generalization. Precision and recall scores in Tab.~\ref{tab:BGG_outlier} further highlight these differences: at 20\% noise level, the Gamma distribution filters a large number of false positives, resulting in a lower precision of around 65\%, whereas the Beta distribution maintains a higher precision of over 80\% compared to both alternatives. This empirically supports our assumption that the Beta distribution is a more appropriate model due to both its flexibility and native unit range support.

\begin{table*}[ht]
\caption{
Comparison of Person Re-ID performance of Beta-SOD against Noise-Aware Re-Id ~\cite{21}, DistributionNet~\cite{19}, PurifyNet~\cite{20}, CORE~\cite{22}, URNet~\cite{URNet}, and TSNT~\cite{23} on the Market1501 and CUHK03 datasets with various random noise levels (5 runs Average and Standard Deviation)}
\resizebox{\textwidth}{!}
{\centering
\begin{tabular}{c|l|cccc|cccc}
\hline
\multirow{2}{*}{Noise} & \multicolumn{1}{c|}{\multirow{2}{*}{Methods}} & \multicolumn{4}{c|}{Market1501} & \multicolumn{4}{c}{CUHK03} \\ \cline{3-10} 
 & \multicolumn{1}{c|}{} & mAP & R1 & R5 & R10 & mAP & R1 & R5 & R10 \\ \hline
\multirow{8}{*}{0\%} & DistributionNet & 70.8 & 87.3 & 94.7 & 96.7 & 38.5 & 39.4 & 58.9 & 67.9 \\
 & URNet & 71.2 & 87.9 & 95.7 & 97.4 & 43.2 & 45.7 & 64.4 & 74.6 \\
 & PurifyNet & 72.1 & 88.4 & 95.8 & 97.6 & 44.8 & 45.8 & 65.2 & 74.2 \\
 & CORE & 74.6 & 89.6 & 96.4 & 98.1 & 46.2 & 46.3 & 65.5 & 75.2 \\
 & TSNT & 80 & 92.7 & 97.1 & 98.2 & 54.3 & 55 & 73.4 & 81.3 \\
 & Noise-Aware Re-Id & 84.4 & 94.4 & 98.1 & 99 & 50.6 & 51.7 & 72.1 & 81 \\ \cline{2-10} 
 & Siamese without OD & \textbf{85.5 $\pm0.9$} & 94.44  $\pm0.62$ & \textbf{98.81 $\pm0.32$} & 98.57 $\pm0.02$ & \textbf{57.1 $\pm1.32$} & \textbf{59.57 $\pm0.93$} & 74.11 $\pm0.81$ & \textbf{81.52 $\pm0.64$} \\
 & Beta-SOD & 84.5 $\pm1.2$ & \textbf{94.51 $\pm0.71$} & 98.35 $\pm0.32$ & \textbf{98.83 $\pm0.02$} & 56.56 $\pm1.5$ & 59.2 $\pm1.1$ & \textbf{74.58 $\pm0.91$} & 79.82 $\pm0.72$ \\ \hline
\multirow{6}{*}{10\%} & DistributionNet & 63.36 & 89.85 & 96.25 & 98.02 & 31.8 & 32.29 & 51.81 & 61.67 \\
 & PurifyNet & 64.3 & 84.2 & 93.7 & 95.6 & 32.8 & 32.8 & 54.8 & 65 \\
 & CORE & 67.7 & 85.5 & 94 & 96.4 & 39.6 & 40.4 & 60.8 & 70.9 \\
 & Noise-Aware Re-Id & 69.9 & 88.8 & 95.4 & 97.4 & 42.4 & 43.5 & 64.8 & 74 \\ \cline{2-10} 
 & Siamese without OD & 59.21 $\pm1.12$ & 75.92 $\pm0.91$ & 89.34 $\pm0.51$ & 92.11 $\pm0.05$ & 29.01 $\pm2.2$ & 32.44 $\pm1.2$ & 48.1 $\pm1.1$ & 60.91 $\pm1$ \\
 & Beta-SOD & \textbf{80.36 $\pm1.43$} & \textbf{92.26 $\pm0.92$} & \textbf{97 $\pm0.49$} & \textbf{98 $\pm0.03$} & \textbf{48.33 $\pm1.9$} & \textbf{52.61 $\pm1.3$} & \textbf{71.82 $\pm0.95$} & \textbf{77.75 $\pm0.77$} \\ \hline
\multirow{8}{*}{20\%} & DistributionNet & 53.4 & 77 & 90.6 & 94 & 24.2 & 24.32 & 43.1 & 53.1 \\
 & URNet & 52.9 & 76.0 & 90.0 & 93.4 & 21.9 & 23.1 & 41.6 & 52.1 \\
 & PurifyNet & 63.1 & 83.1 & 93.3 & 95.9 & 29.2 & 30.3 & 49.8 & 58.4 \\
 & CORE & 66.2 & 84.1 & 93.1 & 95.5 & 35 & 34.4 & 55 & 64.1 \\
 & TSNT & 76 & 90.3 & 96.4 & \textbf{97.7} & \textbf{48.3} & 49.4 & 67.9 & 76.4 \\
 & Noise-Aware Re-Id & 65.5 & 86.2 & 95.7 & 97.6 & 29.2 & 30.9 & 50.6 & 62 \\ \cline{2-10} 
 & Siamese without OD & 50.1 $\pm1.51$ & 72.71 $\pm1.1$ & 85.55 $\pm1$ & 92 $\pm1$ & 23.21 $\pm2.3$ & 25.3 $\pm1.5$ & 40.58 $\pm1.25$ & 50.01 $\pm1.21$ \\
 & Beta-SOD & \textbf{78.32 $\pm1.62$} & \textbf{91 $\pm1.39$} & \textbf{97.11 $\pm0.61$} & 97.61 $\pm0.10$ & 48.12 $\pm2.1$ & \textbf{51.25 $\pm1.32$} & \textbf{71.11 $\pm1$} & \textbf{77.75 $\pm0.78$} \\ \hline
\multirow{5}{*}{30\%} & PurifyNet & 60.2 & 81.4 & 92.3 & 94.9 & 26.4 & 26.6 & 44.6 & 54.2 \\
 & CORE & 60.6 & 81.2 & 92.2 & 94.5 & 34.7 & 34.6 & 56.4 & 64.6 \\
 & Noise-Aware Re-Id & 54.5 & 79.1 & 92.5 & 95.2 & 17.3 & 18.1 & 34.8 & 44.1 \\ \cline{2-10} 
 & Siamese without OD & 44.66 $\pm1.97$ & 65.15 $\pm1.5$ & 78.83 $\pm1.1$ & 82 $\pm1.13$ & 15.71 $\pm2.5$ & 19.19 $\pm1.51$ & 30.84 $\pm1.31$ & 42.4 $\pm1.27$ \\
 & Beta-SOD & \textbf{66.01 $\pm1.91$} & \textbf{85.32 $\pm1.52$} & \textbf{93.26 $\pm1.1$} & \textbf{95.64 $\pm0.27$} & \textbf{46.8 $\pm2.1$} & \textbf{49.88 $\pm1.29$} & \textbf{65.71 $\pm1.1$} & \textbf{76.9 $\pm0.78$} \\ \hline
\end{tabular}}
{\noindent \small \qquad $^*$We show only the available benchmark results as reported in the corresponding papers}
\label{table:person_ReId}
\end{table*}

\begin{table*}[!t]
\centering
\caption{Person Re-ID performance of the OD method using Gaussian, Gamma, and Beta mixture models on the CUHK03 dataset with various random noise levels (5 runs Average)}
\begin{tabular}{c|cccc|cccc|cccc}
\hline
\multirow{2}{*}{Noise} & \multicolumn{4}{c|}{Gaussian mixtures} & \multicolumn{4}{c|}{Gamma mixtures} & \multicolumn{4}{c}{Beta mixtures} \\ \cline{2-13} 
 & mAP & R1 & R5 & R10 & mAP & R1 & R5 & R10 & mAP & R1 & R5 & R10 \\ \hline
0\% & 50.2 & 57.33 & 71.21 & 78.36 & 46.92 & 54.41 & 65.22 & 71.29 & 56.56 & 59.2 & 74.58 & 79.82 \\
10\% & 44.74 & 47.32 & 63.39 & 74.9 & 38.21 & 50.1 & 60.2 & 68.5 & 48.33 & 52.61 & 71.82 & 77.75 \\
20\% & 43.84 & 45.68 & 63 & 73.7 & 34.41 & 42.4 & 54.1 & 60.85 & 48.12 & 51.25 & 71.11 & 77.75 \\ \hline
\end{tabular}
\label{tab:BGG_ReId}
\end{table*}

\begin{table*}[!t]
\caption{Comparison of Person Re-ID performance of Beta-SOD against PurifyNet~\cite{20}, CORE~\cite{22}, and TSNT~\cite{23} on the Market1501 and CUHK03 datasets with various {\bf PATTERN NOISE} levels (5 runs Average and Standard Deviation)}
\resizebox{\textwidth}{!}{%
\centering
\begin{tabular}{c|c|cccc|cccc}
\hline
\multirow{2}{*}{Noise} & \multirow{2}{*}{Methods} & \multicolumn{4}{c|}{Market1501} & \multicolumn{4}{c}{CUHK03} \\ \cline{3-10} 
 &  & mAP & R1 & R5 & R10 & mAP & R1 & R5 & R10 \\ \hline
\multirow{4}{*}{10\%} & PurifyNet & 63.2 & 81.8 & 92.6 & 95.4 & 32.9 & 33.6 & 56.4 & 66 \\
 & CORE & 66.4 & 84 & 93.8 & 96.1 & 32.2 & 38.2 & 54.9 & 64.8 \\\cline{2-10}
 &Beta-SOD & \textbf{68.57} $\pm$ 0.6 & \textbf{87.21} $\pm$ 0.07 & \textbf{95.23} $\pm$ 0.08 & \textbf{97.22} $\pm$ 0.05 & \textbf{33.92} $\pm$ 0.53 & \textbf{41.65} $\pm$ 0.06 & \textbf{59.29} $\pm$ 0.07 & \textbf{69.66} $\pm$ 0.06 \\ \hline
\multirow{4}{*}{20\%} & PurifyNet & 56.2 & 77.8 & 91.1 & 93.8 & 27.9 & 28.8 & 47.3 & 58.2 \\
 & CORE & 62.8 & 81.9 & 92.7 & 95.5 & 33.2 & 34.6 & 54.6 & 64.9 \\
 & TSNT & 74.8 & 89 & 95.3 & 97.2 & 37.2 & 38 & 57.4 & 68.8 \\ \cline{2-10} 
 & Beta-SOD & \textbf{76.89} $\pm$ 0.92 & \textbf{89.58} $\pm$ 0.08 & \textbf{96.82} $\pm$ 0.05 & \textbf{97.67} $\pm$ 0.04 & \textbf{38.1} $\pm$ 0.82 & \textbf{40.31} $\pm$ 0.07 & \textbf{59.55} $\pm$ 0.07 & \textbf{71.08} $\pm$ 0.05 \\ \hline
\end{tabular}}
{\noindent \small \qquad $^*$We show only the available benchmark results as reported in the corresponding papers}
\label{tab:ReId_pattern}
\end{table*}

\vskip 0.15cm
\noindent {\bf Person Re-ID Results.}
We compare our method with several state-of-the-art approaches: Noise-Aware Person Re-Identification~\cite{21}, DistributionNet~\cite{19}, PurifyNet~\cite{20}, CORE~\cite{22}, URNet~\cite{URNet}, and TSTN~\cite{23}. Tab.~\ref{table:person_ReId} presents the comparison under three different noise ratios on CUHK03 and Market1501 datasets. The results demonstrate that the proposed method consistently outperforms the other state-of-the-art methods across the three noise levels. As shown in the table, while the performance of Re-ID deteriorates significantly as the noise level increases from to 30\%, the proposed method maintains comparatively superior performance. We can also observe the relative contribution of the OD module by comparing with the same Siamese architecture without OD. In particular, on the clean dataset the performance is very similar (especially taking into account the standard deviation measured across the 5 re-runs). With added noise, the use of the OD module results in a very significant performance gain.

DistributionNet, originally designed for noisy image classification, filters out noisy samples but can discard significant data when noise exceeds 20\%, harming performance, especially on CUHK03 with limited images per identity. In contrast, Re-ID-specific methods like PurifyNet, CORE and TSNT down-weight noisy samples instead of removing them. However, this still suppresses useful data and degrades performance as noise increases.
Similarly to DistributionNet, our approach adopts a (hard) filtering strategy, adapted for a pair-based setting. Since each image forms multiple cross-view pairs, filtering out noisy pairs still retains sufficient representative samples per identity. This redundancy preserves data diversity and maintains model generalization.

Tab.~\ref{tab:BGG_ReId} demonstrates how the choice of the OD mixture distribution impacts the Re-ID task performance on the CUHK03 dataset at varying random noise levels 0-20\%. Across all noise settings, the Beta distribution consistently outperforms both Gaussian and Gamma models in terms of mAP and Rank-k accuracy. As the noise level increases to 20\%, the Beta maintains its superiority achieving an mAP of 48.12\%, compared to 34-44\% for Gamma and Gaussian mixtures. These results confirm the earlier observation based on Tab.~\ref{tab:BGG_outlier} that the Beta model provides a more appropriate fit for cosine similarities.

Tab.~\ref{tab:ReId_pattern} compares the Re-ID performance of Beta-SOD on the Market1501 and CUHK03 datasets under pattern noise at 10-20\%. 
At 20\% noise level, the relative improvement achieved by Beta-SOD becomes more pronounced. On Market1501, our approach attains 76.89\% mAP and 89.58\% Rank-1 accuracy, surpassing the second best, TSNT, by 2\% mAP and 0.5\% Rank-1, while also maintaining consistently higher Rank-5 and Rank-10 results. Similarly, on the more challenging CUHK03 dateset, the relative improvement over TSNT is of 0.9\% mAP and 2.3\% Rank-1. To provide a measure of statistical relevance, we also report the standard deviations available for our experiments (over 5 re-runs), which confirm the meaningful nature of these improvements.
Overall, the results confirm the robustness of our OD-based framework, which performs well in presence of both random (realistic setting) and pattern (adversarial setting) noise.

\vskip 0.15cm
\noindent 
\textbf{Mislabeled examples.} Fig.~\ref{fig:fp_fn} presents examples of image pairs incorrectly handled by the OD method during denoising. The first row shows false positives, i.e., similar image pairs that were wrongly filtered out by the OD method, with cosine similarity scores below 0.8. The second row shows false negatives, i.e., noisy image pairs that were mistakenly retained as clean, with cosine similarities above 0.6. These examples illustrate the difficulty of distinguishing visually similar pairs, where both under-filtering (false negatives) and over-filtering (false positives) compromise the quality of the denoised dataset. These and conceptually similar challenging examples are treated incorrectly by the OD module even at 0\% noise, see in Tab~\ref{tab:denoising_results}. At the same time, the overall impact achieved by the OD-based filtering in the person Re-ID task ranges from a modest performance reduction of 0.5-1\% mAP at 0\% noise to substantial improvement of 21-31\% mAP at 30\% noise levels, see Tab~\ref{table:person_ReId}.

\begin{figure*}[!t]
    \centering
    \includegraphics[scale=0.44]{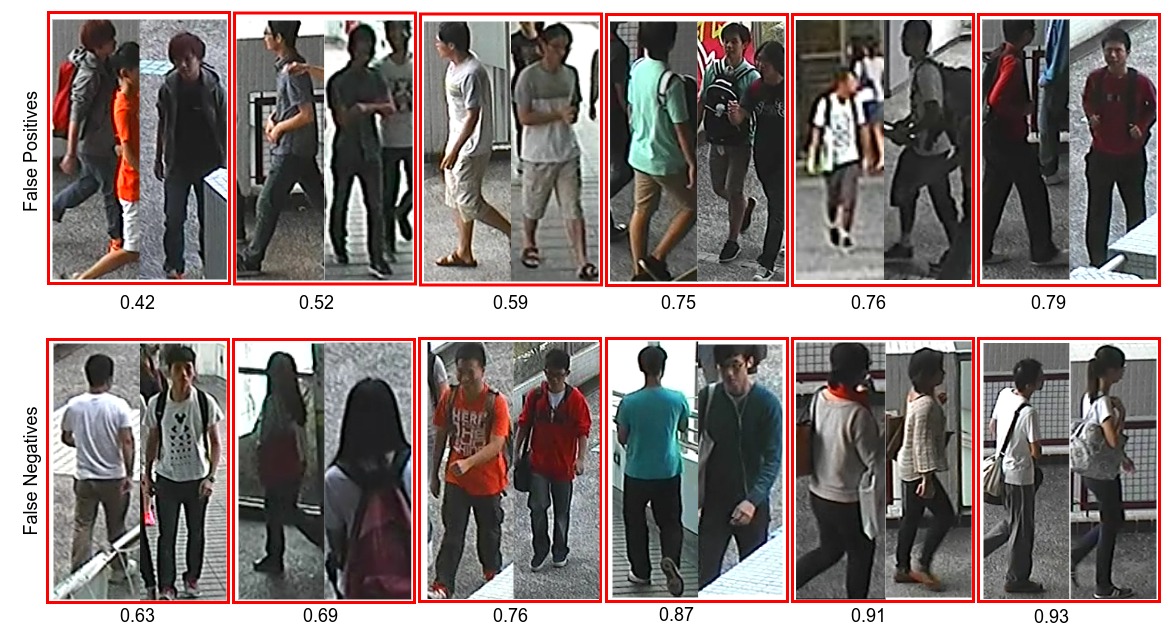}
    \vspace{-\baselineskip}
    \caption{Examples of incorrectly filtered image pairs by Beta-SOD on CUHK03 with 20\% noise: false positives, i.e. true class 1 flagged as 0 (top row), and false negatives, i.e. true class 0 flagged as 1 (bottom row). The numbers are cosine similarities: below 0.8 for false positives and above 0.6 for false negatives.}
    \label{fig:fp_fn}
\end{figure*}

\vskip 0.15cm
\noindent {\bf Vehicle Re-ID Results.}
We now evaluate Beta-SOD on vehicle Re-ID using VeRi-776 dataset. As shown in Tab.~\ref{table:Vehicle_ReId}, our method achieves high mAP and R1/R5 scores on clean dataset (0\% noise), only behind that of specialized~\cite{GLSD,LSKA-ReID}. Interestingly, the configuration without the OD module attains the level not far behind the top performer without attention mechanism incorporated. As established in the de-noising comparisons, without the presence of noise, the OD module removes more challenging training examples which reduces the test performance. This results in a drop of mAP by 1.2\%.

To evaluate noisy dataset performance of, random label noise is introduced into the vehicle Re-ID dataset. Tab.~\ref{table:Vehicle_ReId} reports performance with and without the Beta-SOD. On the clean dataset, the model achieves 83\% mAP, 97.22\% Rank-1, and 98.1\% Rank-5 accuracy. With added noise, performance without the Beta-SOD filtering drops significantly, but is effectively handled with filtering. No benchmark methods report their performance in the presence of noisy labels.

These results confirm the ability of the proposed technique to generalize to another object Re-ID task without modifications.

\begin{table}[t]
\caption{Comparison of Vehicle Re-ID performance on the VeRi-776 dataset with various random noise levels (5 runs Average)}
\resizebox{\columnwidth}{!}{%
\begin{tabular}{c|l|ccc}
\hline
Noise & \multicolumn{1}{c|}{Methods} & mAP & R1 & R5 \\ \hline
\multirow{9}{*}{0\%} & GS-TRE ~\cite{31} & 59.47 & 96.24 & 98.97 \\
 & STE-VReID ~\cite{34} & 67.55 & 90.23 & 96.42 \\
 & GiT ~\cite{git} & 80.3 & 96.9 & - \\
 & SGFD ~\cite{33} & 81 & 96.7 & 98.6 \\
 & GLSD ~\cite{GLSD} & 83.5 & 97.3 & 98.6 \\
 & LSKA-ReID~\cite{LSKA-ReID} & \textbf{85.64} & 97.97 & \textbf{99.04} \\ \cline{2-5}
 & Siamese without OD & 84.2 & \textbf{98.13} & 98.8 \\
 & Beta-SOD & 83 & 97.22 & 98.1 \\ \hline\hline
\multirow{2}{*}{10\%} & Siamese without OD & 56.25 & 88.12 & 91.45 \\
 & Beta-SOD & \textbf{78} & \textbf{97} & \textbf{98} \\ \hline
\multirow{2}{*}{20\%} & Siamese without OD & 52.22 & 84.92 & 90.11 \\
 & Beta-SOD & \textbf{76.26} & \textbf{96.13} & \textbf{97.62} \\ \hline
\multirow{2}{*}{30\%} & Siamese without OD & 49.41 & 83.52 & 88.21 \\
 & Beta-SOD & \textbf{74.21} & \textbf{95.93} & \textbf{97.54} \\ \hline
\end{tabular}}
\label{table:Vehicle_ReId}
\end{table}

\vskip 0.15cm
\noindent {\bf Ablation Analysis.}
Finally, we investigate the impact of the considered loss functions on the model performance when using the CUHK03 dataset. This involves a comparative analysis between a basic classification loss and similarity-based constraints. The experimental results are presented in Tab.~\ref{table:ablation}.

We establish a baseline Siamese architecture comprising the FCNN trained with the $(\mathcal{L}_{CE})$ loss. Using only $\mathcal{L}_{CE}$ our model achieves a mAP of 47.41\% on the clean dataset and 36.12\% in the presence of 20\% label noise.
Adding $\mathcal{L}_{cos}$ to $\mathcal{L}_{CE}$ improves the mAP by just over 4\% on both the clean and the 20\% noisy datasets. Similarly, combining $\mathcal{L}_{c}$ with $\mathcal{L}_{CE}$ results in a mAP improvement of about 3-5\% across both settings. Finally, by combining all three loss functions, as described in Eq.~\ref{eq:combined_loss}, the mAP improves by an additional 2\%.

\begin{table}[t]
\caption{Loss Function Selection on CUHK03 Dataset with random noise level of 0\% and 20\% (5 runs average)}
\resizebox{\columnwidth}{!}{%
\begin{tabular}{ccc|ccc|ccc}
\hline
\multicolumn{3}{c|}{Loss Functions} & \multicolumn{3}{c|}{0\% noise} & \multicolumn{3}{c}{20\% noise} \\ \hline
$\mathcal{L}_{CE}$ & $\mathcal{L}_{COS}$ & $\mathcal{L}_{C}$ & mAP & R1 & R5 & mAP & R1 & R5 \\ \hline
\ding{51}  &  &  & 47.41 & 50.9 & 63.13 & 36.12 & 40.3 & 48.21 \\
\ding{51}  &\ding{51}  &  & 51.82 & 55.3 & 70.22 & 40.2 & 45.01 & 58.3 \\
\ding{51}  &  &\ding{51}  & 54.72 & 57.1 & 73.1 & 45.8 & 50.13 & 68.83 \\
\ding{51}  &\ding{51}  &\ding{51}  & 56.52 & 59.2 & 74.58 & 48.12 & 51.25 & 71.11 \\ \hline
\end{tabular}}
\label{table:ablation}
\end{table}

\section{Conclusion}
We propose a robust object Re-ID approach that integrates a statistical OD framework into training. Our Beta-SOD framework models cosine similarity distributions with a two-component Beta mixture to iteratively filter noisy labels, enabling unsupervised task-specific dataset de-noising. A Siamese network is then trained using combined binary cross-entropy, contrastive, and cosine embedding losses to enhance both classification and feature similarity. Experiments show improved robustness and accuracy across Re-ID benchmarks under varying noise levels. The observed performance gains strongly suggest the effectiveness and generalization capacity of Beta-SOD in real-world noisy label environments.

We present a novel theoretical result demonstrating that the unique identifiability of two-component Beta mixtures is attainable with probability 1. 
Based on the strong performance observed in Re-ID tasks, we expect Beta-mixture modeling to be a versatile and effective tool for de-noising cosine similarity scores produced by Siamese architectures, with potential applications extending to face recognition, speaker verification, and sentence embedding models.

While our statistical OD strategy proved effective as a standalone denoising module, future work will explore integrating this mechanism more tightly into the training neural network pipeline. This can be achieved by embedding the EM-based outlier scores directly into the loss function by re-weighting each training pair's similarity contribution based on its clean class probability $\gamma$ (soft-filtering). 
An additional extension would be to integrate an attention mechanism into the Siamese architecture, allowing the model to emphasize the most discriminative features within paired images. This could lead to more accurate estimation of cosine similarities, thereby enhancing the effectiveness and efficiency of the similarity-based OD.

\bibliographystyle{IEEEtran}
\bibliography{egbib.bib}
\end{document}